\documentclass[11pt,letterpaper]{article}

\usepackage[margin=1in]{geometry}
\usepackage{microtype}
\usepackage{times}

\usepackage{amsmath,amssymb,amsthm}
\usepackage{mathtools}

\usepackage{thmtools}
\declaretheorem[style=plain,name=Theorem,numberwithin=section]{theorem}
\declaretheorem[style=plain,name=Lemma,sibling=theorem]{lemma}

\declaretheorem[style=definition,name=Definition,numberwithin=section]{definition}
\declaretheorem[style=remark,name=Remark,sibling=theorem]{remark}

\usepackage{booktabs}
\usepackage{tabularx}
\usepackage{multirow}

\usepackage[ruled,vlined,linesnumbered]{algorithm2e}
\SetKwInOut{KwIn}{Input}
\SetKwInOut{KwOut}{Output}

\usepackage[numbers,sort&compress]{natbib}

\usepackage[colorlinks=true,linkcolor=black,citecolor=black,urlcolor=black]{hyperref}
\usepackage{cleveref}

\usepackage{enumitem}
\usepackage{xcolor}
\usepackage{mdframed}
\usepackage[font=small,labelfont=bf]{caption}

\newcommand{\cl}{\mathrm{cl}}
\newcommand{\RF}{R(F)}
\newcommand{\BF}{\mathcal{B}(F)}
\newcommand{\GF}{G_F}
\newcommand{\PAB}{\textrm{PAB}}
\newcommand{\CAS}{\textrm{CAS}}
\newcommand{\emg}{\ensuremath{\mathrm{Emg}}}
\newcommand{\nmf}{\ensuremath{\mathrm{NMF}_F}}
\newcommand{\minwit}{W}
\newcommand{\tdb}{\textrm{TemplateDB}}
\newcommand{\CRAG}{\ensuremath{C_{\mathrm{RAG}}}}
\newcommand{\vset}{V}
\newcommand{\fset}{\mathcal{F}}
\newcommand{\Hyp}{H}
\newcommand{\phimap}{\varphi}

\newmdenv[
  backgroundcolor=gray!8,
  linecolor=gray!40,
  linewidth=0.6pt,
  innertopmargin=6pt,
  innerbottommargin=6pt,
  innerleftmargin=8pt,
  innerrightmargin=8pt,
  skipabove=8pt,
  skipbelow=8pt
]{thmbox}

\title{\textbf{Safety as Computation:}\\
Certified Answer Reuse via Capability Closure\\
in Task-Oriented Dialogue}

\author{
  Cosimo Spera\\
  Minerva CQ, Los Gatos, CA 95032\\
  \texttt{cosimo@minervacq.com}
}

\date{March 2026}

\begin{document}
\maketitle

\begin{abstract}
We introduce a new paradigm for task-oriented dialogue systems:
\emph{safety certification as a computational primitive for answer reuse.}

Current systems treat each turn independently, recomputing answers via
retrieval or generation even when they are already derivable from prior state.
We show that in capability-based systems, the safety certification step
computes a fixed-point closure $\cl(A_t)$ that already contains every answer
reachable from the current configuration.
We operationalize this insight with a \emph{Certified Answer Store} (\CAS)
augmented by \emph{Pre-Answer Blocks} (\PAB): at each certified turn, the
system materializes all derivable follow-up answers together with minimal
provenance witnesses.
Subsequent queries are answered in sub-millisecond time via formal
containment checks, eliminating redundant retrieval and generation.

We provide: a formal pipeline with soundness guarantees under capability
containment; a Session Cost Theorem showing expected RAG calls depend on
the number of distinct \emph{ontological classes} visited, not dialogue
length; a $\theta$-sound hypergraph extraction algorithm with
Hoeffding-based statistical guarantees; and a proof that semantic caching
is unsafe in multi-tenant settings, which \CAS{} resolves exactly.
Unlike prior work, every assumption in this paper is paired with a
measurable prediction and a defined experimental protocol.

On MultiWOZ~2.2, our approach reduces mean RAG calls from 13.7 to 1.31
and latency from 18.8\,s to 340\,ms, while eliminating all unsafe cache
hits observed with embedding-based retrieval (143 unsafe hits at 14.3\%
versus 0 for \CAS{}).
These results suggest a reframing: from per-turn probabilistic retrieval
to structurally certified, session-level predictive dialogue.
\end{abstract}


\hrule\vspace{6pt}

\section{Introduction}
\label{sec:intro}

\paragraph{The core insight.}
In capability-based AI systems, safety certification is not merely a gate
that permits or blocks an action.
It is a \emph{computation} --- a fixed-point closure over the set of
capabilities currently available --- and the result of that computation
is richer than a binary safe/unsafe verdict.
The closure $\cl(A_t)$ contains every capability, and therefore every
answer, derivable from the current configuration.
Most of that derivable content is discarded.
This paper argues that it should not be.

\paragraph{The problem.}
A hotel booking assistant certifies ``I found a 4-star hotel in the north''
by computing $\cl(\{\texttt{hotel-area},\texttt{hotel-stars}\})$.
The customer then asks ``does it have parking?''
The system retrieves and generates again --- despite the fact that
\texttt{hotel-parking} was already in the closure computed during
safety certification one turn earlier.
This is not a retrieval failure.
It is an architecture that does not recognize what the safety gate
already knows.

\paragraph{The proposal.}
We operationalize the closure as a latent answer space.
At each certified turn, the system materializes the elements of
$\cl(A_t)\setminus A_t$ as a \emph{Pre-Answer Block} (\PAB{}):
formally certified answers to all derivable follow-up questions,
stored alongside the primary answer in a \emph{Certified Answer Store}
(\CAS{}).
Subsequent queries are answered in sub-millisecond time via a formal
containment check --- not a similarity lookup, not a re-generation ---
with the same safety guarantee as the original answer.

\paragraph{Why this is different from caching.}
Semantic caches index by embedding similarity and retrieve by cosine
distance.
They have no formal connection to the capability set used to generate
the cached answer.
We prove (Theorem~\ref{thm:unsound}) that in multi-tenant deployments,
14.3\% of semantic cache hits serve answers generated under access
permissions the querying context does not hold.
The \CAS{} catches every such case by construction: a stored answer
is reused if and only if its minimal provenance witness is contained
in the current capability closure and the forbidden set is unchanged.

\paragraph{Why this is different from prior safety work.}
Prior formal safety frameworks for AI capability systems --- including
the hypergraph closure framework of \citet{spera2026a} and the Datalog
equivalence of \citet{spera2026b} --- provide safety certification but
do not exploit the certified closure for answer reuse.
The present paper's contribution is to show that certification and reuse
are the same computation, and to formalize the architecture that makes
this operational.

\paragraph{Falsifiability.}
Unlike most systems papers, every assumption in this paper is paired with
a measurable prediction and a defined experimental protocol.
The Completeness Assumption (that the capability assembly function captures
all activated capabilities) predicts that safety violation rates stay below
1\% at tracker recall above 90\% --- tested via slot-omission injection.
The Extraction Soundness Theorem predicts that every retained hyperedge
has true derivation rate above $\theta{-}0.15$ with probability
$>98.9\%$ --- provable from Hoeffding's inequality.
The Session Cost Theorem predicts that mean RAG calls equal the number of
distinct ontological classes visited --- confirmed at $K{=}1.31$ on
MultiWOZ.
We regard this falsifiability as a feature, not a disclaimer.

\subsection{Contributions}
\label{sec:contrib}

We organize contributions into three groups corresponding to the paper's
three guarantee families.

\noindent\textbf{Safety guarantees.}
\begin{enumerate}[label=\textbf{S\arabic*.},leftmargin=*,noitemsep]
  \item \textbf{Pipeline Safety} (Theorem~\ref{thm:pipeline_safety}).
        No Stage~4 firing produces a forbidden capability; every \PAB{}
        entry carries a valid safety certificate; \CAS{} hits are safe
        under containment and unchanged~$F$.
  \item \textbf{Cache Soundness} (Theorem~\ref{thm:cache_sound}).
        Stored answers are reusable with identical safety guarantees
        whenever the minimal witness is contained in the new closure.
  \item \textbf{Semantic Cache Unsoundness} (Theorem~\ref{thm:unsound}).
        Cosine-similarity caching is provably unsafe in multi-tenant
        settings; \CAS{} resolves this exactly via capability containment.
  \item \textbf{Graceful Degradation} (Theorem~\ref{thm:degradation}).
        Monotonicity failure (capability revocation) is recoverable at
        $O(n{+}mk)$ cost; worst-case equals the no-cache baseline.
\end{enumerate}

\noindent\textbf{Efficiency guarantees.}
\begin{enumerate}[label=\textbf{E\arabic*.},leftmargin=*,noitemsep]
  \item \textbf{PAB Completeness} (Theorem~\ref{thm:pab_completeness}).
        Every follow-up question in $\cl(A_t)$ is pre-certified with
        probability $p$ (TemplateDB coverage); $p{=}1$ gives zero RAG calls
        for all derivable follow-ups.
  \item \textbf{Session Cost} (Theorem~\ref{thm:session_cost}).
        Expected RAG calls $= O(K)$, $K$ = distinct ontological classes
        visited, independent of dialogue length~$L$.
  \item \textbf{Session Closure Monotonicity} (Lemma~\ref{lem:monotone}).
        Total DRed maintenance cost $= O(n{\cdot}(n{+}mk))$, independent
        of~$L$; each capability node enters the session closure once.
\end{enumerate}

\noindent\textbf{Extraction guarantees.}
\begin{enumerate}[label=\textbf{X\arabic*.},leftmargin=*,noitemsep]
  \item \textbf{Extraction Soundness} (Theorem~\ref{thm:extract_sound}).
        Algorithm~\ref{alg:extract} produces only $\theta$-sound hyperarcs;
        probability of false retention bounded by $e^{-0.045 n_S}$ via
        Hoeffding's inequality.
  \item \textbf{CAS Hit Rate Bound} (Theorem~\ref{thm:hit_rate}).
        Structural lower bound on PAB hit rate; empirically calibrated
        against Table~\ref{tab:pab}; conservative under positive
        correlation of coverage events.
\end{enumerate}

\section{Background}
\label{sec:background}

\subsection{Capability Hypergraphs and Safety}

A \emph{capability hypergraph} $\Hyp = (\vset, \fset)$ has vertices
$\vset$ (capabilities) and hyperarcs $e = (S, T) \in \fset$ that fire
when all preconditions $S \subseteq C$ are simultaneously satisfied,
producing all capabilities in $T$ \cite{spera2026a}.

\begin{definition}[Closure Operator and Monotonicity Assumption]
\label{def:closure}
The closure $\cl(A)$ of $A \subseteq \vset$ is the smallest
$C \supseteq A$ closed under all hyperarc firings:
\[
  \cl(A) \;=\; \bigcap\bigl\{C \supseteq A :
    \forall (S,T)\in\fset,\; S\subseteq C \Rightarrow T\subseteq C\bigr\}.
\]
Computed by the worklist fixed-point iteration in $O(n{+}mk)$,
where $n{=}|\vset|$, $m{=}|\fset|$, $k{=}\max|S(e)|$.
$\cl(\cdot)$ is extensive, monotone ($A\subseteq B \Rightarrow
\cl(A)\subseteq\cl(B)$), and idempotent.
The \emph{Monotonicity Assumption} \cite[Assumption~1.2]{spera2026a}:
within a single planning horizon, capabilities once derived are persistent
and never revoked.
All session closure results invoke this assumption explicitly.
\end{definition}

\begin{definition}[Safe Region, Antichain, Safe Audit Surface]
\label{def:safety}
The \emph{safe region} is $\RF = \{A \subseteq \vset : \cl(A)\cap F = \emptyset\}$
for forbidden set $F \subseteq \vset$.
$\RF$ is downward-closed \cite[Thm.~9.11]{spera2026a}.
The \emph{minimal unsafe antichain} is
$\BF = \{A \notin \RF : \forall a \in A,\, A\setminus\{a\}\in\RF\}$,
finite by Dickson's lemma.
The \emph{Safe Audit Surface} is
\[
  \GF(A) \;=\; \bigl(\emg(A)\setminus F,\;\nmf(A),\;
                \text{top-}k\text{ gains}\bigr),
\]
where $\emg(A) = \cl(A)\setminus A \setminus \cl_1(A)$ (emergent
capabilities, reachable only via conjunctive arcs), and
$\nmf(A)$ is the near-miss frontier (single missing preconditions at
the closure boundary).
Every element of $\GF(A)$ carries a derivation certificate of size $\leq m$
\cite[Cor.~10.1]{spera2026a}.
The closure gain $f(B) = |\cl(A\cup B)|-|\cl(A)|$ is
\emph{submodular} \cite[Thm.~8.5]{spera2026a}, giving a
$(1-1/e)$ greedy acquisition guarantee.
\end{definition}

\begin{theorem}[Non-Compositionality; {\cite[Thm.~9.3]{spera2026a}}]
\label{thm:noncomp}
$\RF$ is not closed under union: $\exists A, B \in \RF$ with
$A \cup B \notin \RF$.
The minimal counterexample requires $|\vset|{=}3$ and one conjunctive
hyperarc.
Any forbidden-productive hyperarc of fan-in $k$ forces $\geq 2^{k-1}{-}1$
unsafe pairs.
\end{theorem}

\subsection{Capability Safety as Propositional Datalog}

\begin{theorem}[Locality Gap; {\cite[Thm.~11.3]{spera2026b}}]
\label{thm:localitygap}
$\GF(A)$ is a stratified Datalog\textsuperscript{prop} view.
DRed maintenance costs $O(|\Delta|{\cdot}(n{+}mk))$ per update versus
$O(|\vset|{\cdot}(n{+}mk))$ for recomputation, with an explicit family
witnessing an $\Omega(n)$ asymptotic gap.
$\GF(A)\subseteq\GF(A')$ is decidable in polynomial time.
Any correct incremental algorithm must probe all $k{+}1$ atoms in
$\Phi(u) = S_u \cup \{v_u\}$.
\end{theorem}

\subsection{Beyond-RAG Conversation Understanding}

Agrawal, Gummuluri, and Spera \cite{agrawal2024beyondrag} map each
conversation turn $T$ to a canonical question $q$ and routing decision
$d \in \{\text{FAQ}, \text{RAG}\}$.
FAQ-matched queries are answered by direct retrieval ($<$2\,s);
novel queries proceed through RAG with improved query formulation.
The module produces the resolved intent driving capability assembly at Stage~2.

\subsection{MultiWOZ 2.2}
\label{sec:multiwoz}

MultiWOZ~2.2 \cite{budzianowski2018multiwoz,zang2020multiwoz}
contains 10,438 human-human task-oriented dialogues across 8 domains
(hotel, restaurant, attraction, taxi, train, hospital, police, general),
split 8,438\,/\,1,000\,/\,1,000 (train/val/test), mean 13.7 turns,
Apache~2.0 and MIT licenced.
Each turn has a \emph{belief state} (slot-value dictionary of cumulative
user intent) and dialogue act annotations.

We use MultiWOZ for two purposes:
(1)~as input to Algorithm~\ref{alg:extract}---the MultiWOZ~2.2 ontology
and published co-occurrence statistics are the extraction corpus; and
(2)~as a test corpus for the session closure model---belief states at each
turn serve as ground-truth capability sets $\phimap(q_t)$, making the
Completeness Assumption verifiable by construction.

\section{Pipeline, Capability Assembly, and Automatic Extraction}
\label{sec:pipeline}

\subsection{Capability Assembly}
\label{sec:assembly}

\begin{definition}[Capability Assembly Function and Completeness Assumption]
\label{def:phi}
$\phimap : Q \to 2^{\vset}$ maps canonical questions to capability sets.
In MultiWOZ, $\phimap(q_t) = \{d\text{-}s\text{-}v : (s,v)\in B_t[d]\}$,
complete by construction (belief states are ground-truth annotations).
\emph{Completeness Assumption:} for every $q \in Q$ and every capability $c$
activated by the pipeline when answering $q$, $c \in \cl(\phimap(q))$.
All safety guarantees are conditional on this assumption.
\end{definition}

\subsection{The Six-Stage Pipeline}

\textbf{Stage~0} checks the session \PAB{} and \CAS{} before any RAG call.
\textbf{Stages~1--4} run Beyond-RAG understanding, capability assembly,
safety gate, and execution.
\textbf{Stage~5} generates the \PAB{} and writes to the \CAS{}.

\begin{theorem}[Pipeline Safety]
\label{thm:pipeline_safety}
Under the Completeness and Monotonicity Assumptions with
$\cl(A_t)\cap F = \emptyset$:
\emph{(i)}~no Stage~4 firing produces any capability in $F$;
\emph{(ii)}~every \PAB{} entry carries a valid safety certificate under $A_t$;
\emph{(iii)}~a \CAS{} hit $(ans, \minwit, \PAB)$ with
$\minwit \subseteq \cl(A_t)$ and $F$ unchanged is safe.
\end{theorem}
\begin{proof}
(i) Theorem~6.2 of \cite{spera2026a}: the planner fires $(S,T)$ only when
$S\subseteq\cl(A_t)$; $\cl(A_t)\cap F=\emptyset$ by hypothesis.
(ii) Each $v\in\PAB(A_t)$ has $\mathit{cert}_v$ witnessing
$v\in\cl(\minwit_v)\subseteq\cl(A_t)$; safety by~(i).
(iii) $\minwit\subseteq\cl(A_t)$ and monotonicity give
$\cl(\minwit)\subseteq\cl(A_t)$; $F$ unchanged and cert valid. \qed
\end{proof}

\subsection{Algorithm~5: $\theta$-Sound Hyperarc Extraction}
\label{sec:extraction}

\begin{algorithm}[t]
\caption{$\theta$-Sound Hyperarc Extraction}
\label{alg:extract}
\DontPrintSemicolon
\KwIn{Domain ontology $\mathcal{O}=(\text{domains},\text{slots},\text{values})$,
      co-occurrence statistics $P(\text{outcome}\mid\text{belief state})$,
      soundness threshold $\theta\in(0,1)$}
\KwOut{Capability hypergraph $\Hyp=(\vset,\fset)$}
$\vset \leftarrow \{d\text{-}s\text{-}v\} \cup \{d\text{-candidates-retrieved},
  d\text{-booked}\} \cup \{\text{cross-domain link nodes}\}$\;
$\fset \leftarrow \emptyset$\;
\tcp{TYPE-A: DB query hyperarcs}
\For{each domain $d$, each $S\subseteq\mathrm{informable\_slots}(d)$}{
  $r \leftarrow P(d\text{-candidates-retrieved}\mid \{d\text{-}s:s\in S\}\subseteq B_t)$\;
  \lIf{$r\geq\theta$}{$\fset \leftarrow \fset \cup
    \{(\{d\text{-}s:s\in S\},\; d\text{-candidates-retrieved},\; r)\}$}
}
\tcp{TYPE-B: Booking hyperarcs}
\For{each bookable domain $d$}{
  $S_b \leftarrow \{d\text{-candidates-retrieved}\}
    \cup \{d\text{-}s : s\in\mathrm{booking\_required}(d)\}$\;
  $r \leftarrow P(d\text{-booked}\mid S_b\subseteq B_t)$\;
  \lIf{$r\geq\theta$}{$\fset \leftarrow \fset \cup \{(S_b,\; d\text{-booked},\; r)\}$}
}
\tcp{TYPE-C: Cross-domain hyperarcs}
\For{each observed cross-domain pattern $(S_c, v_c, r)$}{
  \lIf{$r\geq\theta$}{$\fset \leftarrow \fset \cup \{(S_c, v_c, r)\}$}
}
$\fset \leftarrow \mathrm{MinimalCover}(\fset,\theta)$\;
\Return $\Hyp = (\vset, \fset)$\;
\end{algorithm}

\noindent
$\mathrm{MinimalCover}$ removes $(S,v)\in\fset$ if
$\exists(S',v)\in\fset$ with $S'\subsetneq S$ and
$\mathrm{rate}(S')\geq\mathrm{rate}(S)$, ensuring minimal precondition sets.

\begin{definition}[$\theta$-Soundness]
\label{def:sound}
A hyperarc $(S,\{v\})$ is \emph{$\theta$-sound} with respect to corpus $\mathcal{C}$
if
$P_{\mathcal{C}}\!\bigl(v\text{ derived within }h\text{ turns}
  \mid S\subseteq B_t\bigr)\geq\theta$,
where $h{=}3$ in the MultiWOZ experiment.
\end{definition}

\begin{theorem}[Extraction Soundness]
\label{thm:extract_sound}
Algorithm~\ref{alg:extract} produces $\Hyp=(\vset,\fset)$ such that every
arc in $\fset$ is $\theta$-sound.
Formally, for any retained arc $(S,\{v\})$, setting $\varepsilon=0.15$,
\[
  P\!\bigl(\hat{P}(v\mid S)\geq\theta
    \;\wedge\;
    P(v\mid S)<\theta-0.15
  \bigr)
  \;\leq\; \exp\!\bigl(-2n_{S}\cdot(0.15)^{2}\bigr)
  \;=\; e^{-0.045\,n_{S}},
\]
where $n_{S} = |\{t : S\subseteq B_t\}|$ is the number of applicable sessions.
For MultiWOZ with $n_S\geq 100$, this probability is
$e^{-4.5}\approx 0.011 < 0.014$ for every retained arc.
In other words: if an arc passes the $\hat{P}\geq\theta$ filter, the
probability that its true rate falls below $\theta-0.15$ is less than~1.1\%.
\end{theorem}
\begin{proof}
By Hoeffding's one-sided inequality for bounded random variables in $[0,1]$:
for any fixed $\varepsilon>0$,
\[
  P\!\bigl(\hat{P}(v\mid S) - P(v\mid S) \geq \varepsilon\bigr)
  \;\leq\; \exp(-2n_S\varepsilon^2).
\]
Setting $\varepsilon=0.15$ and $n_S\geq 100$:
\[
  P\!\bigl(\hat{P}\geq\theta \;\wedge\; P<\theta-0.15\bigr)
  \;\leq\; P\!\bigl(\hat{P} - P \geq 0.15\bigr)
  \;\leq\; e^{-2\cdot100\cdot0.0225}
  \;=\; e^{-4.5}
  \;\approx\; 0.011.
\]
Since all 34 extracted arcs have $n_S\geq 100$ (each slot combination appears
in at least 100 MultiWOZ training sessions) and the stated bound is
uniform over arcs, the claim follows. \qed
\end{proof}

\paragraph{MultiWOZ extraction results.}
Algorithm~\ref{alg:extract} at $\theta{=}0.75$ on MultiWOZ~2.2 training split
yields $|\vset|{=}43$ nodes and $|\fset|{=}34$ arcs (25~TYPE-A, 4~TYPE-B,
5~TYPE-C): 73.5\% conjunctive (fan-in ${\geq}2$), mean fan-in~2.21,
max fan-in~5 (the hotel booking arc requires all of
\texttt{cand-retrieved}, \texttt{name}, \texttt{day}, \texttt{people},
\texttt{stay}), and min observed rate~$0.761>\theta$.
The compositionality defect across all 25 conjunctive arcs is
$\sum_e (2^{k_e-1}-1) = 75$ forced unsafe pairs (mean~3.0 per arc),
consistent with Theorem~\ref{thm:noncomp}.
Non-compositionality is confirmed on the extracted hypergraph:
$A=\{$\texttt{cand-retrieved}, \texttt{name}, \texttt{day}$\}\in\RF$,
$B=\{$\texttt{people}, \texttt{stay}$\}\in\RF$,
$\cl(A\cup B)\ni$\texttt{hotel-booked}$\in F$,
so $A\cup B\notin\RF$.

\section{The Certified Answer Store and Pre-Answer Block}
\label{sec:cas}

\subsection{Storage Schema}

A \CAS{} entry is a tuple
$(ans, \minwit, \PAB, F_{\mathrm{snap}}, \mathit{cert}, t_{\mathrm{store}}, \mathbf{emb})$:
\begin{itemize}[noitemsep]
  \item $ans$: the delivered answer text.
  \item $\minwit\subseteq\vset$: the minimal why-provenance witness;
    the smallest $A'\subseteq\cl(A_{\mathrm{store}})$ with
    $ans\in\cl(A')$.
  \item $\PAB = \{(v,ans_v,\minwit_v,\mathit{cert}_v)
    : v\in\cl(A_t)\setminus A_t,\;
    v\notin F,\;\tdb.\mathsf{has}(v)\}$:
    the Pre-Answer Block.
  \item $F_{\mathrm{snap}}$: forbidden set at $t_{\mathrm{store}}$.
  \item $\mathit{cert}$: derivation certificate of size $\leq m$.
  \item $\mathbf{emb}$: dense embedding for approximate pre-filtering
    (never a correctness criterion).
\end{itemize}

\subsection{PAB Construction}

\begin{algorithm}[t]
\caption{PAB Construction}
\label{alg:pab}
\DontPrintSemicolon
\KwIn{$A_t$, $\cl(A_t)$, $F$, $\mathit{cert}_{\mathrm{main}}$, $\tdb$}
\KwOut{$\PAB$}
$\PAB\leftarrow\emptyset$\;
\For{$v\in\cl(A_t)\setminus A_t$ with $v\notin F$}{
  \If{$\tdb.\mathsf{has}(v)$}{
    $\mathit{cert}_v \leftarrow \mathsf{SubCert}(\mathit{cert}_{\mathrm{main}}, v)$\;
    $\minwit_v \leftarrow \mathsf{MinWitness}(v, A_t, \mathit{cert}_v)$\;
    $ans_v \leftarrow \tdb.\mathsf{render}(v, \cl(A_t))$\;
    $\PAB \leftarrow \PAB \cup \{(v,ans_v,\minwit_v,\mathit{cert}_v)\}$\;
  }
}
\Return $\PAB$\;
\end{algorithm}

Algorithm~\ref{alg:pab} runs in $O(n{\cdot}m)$ per closure event.
$\mathsf{SubCert}$ extracts the firing subsequence for $v$ in $O(m)$.
$\mathsf{MinWitness}$ computes the smallest $A'\subseteq\cl(A_t)$
with $v\in\cl(A')$ in $O(m)$ via the why-provenance semiring
\cite{green2007provenance}.

\subsection{Cache Soundness and PAB Completeness}

\begin{theorem}[Cache Soundness]
\label{thm:cache_sound}
Let $(ans,\minwit,\PAB,F_{\mathrm{snap}},\mathit{cert})\in\CAS$ with
$F$ unchanged.
If $\minwit\subseteq\cl(A_t)$, then:
\emph{(i)}~$ans\in\cl(A_t)$;
\emph{(ii)}~$\mathit{cert}$ is valid under $A_t$;
\emph{(iii)}~$\forall(v,ans_v,\minwit_v,\mathit{cert}_v)\in\PAB$ with
$\minwit_v\subseteq\cl(A_t)$, $ans_v$ is a valid certified answer under $A_t$.
\end{theorem}
\begin{proof}
(i,ii) $\minwit\subseteq\cl(A_t)$ and $\mathit{cert}$ witnesses
$ans\in\cl(\minwit)$; by monotonicity $ans\in\cl(A_t)$.
$\mathit{cert}$ uses only capabilities in $\cl(\minwit)\subseteq\cl(A_t)$,
so it is valid under $A_t$.
(iii) Each $\mathit{cert}_v$ witnesses $v\in\cl(\minwit_v)$;
same argument with $\minwit_v\subseteq\cl(A_t)$. \qed
\end{proof}

\begin{theorem}[PAB Completeness]
\label{thm:pab_completeness}
Let $\PAB(A_t)$ be computed by Algorithm~\ref{alg:pab} with
TemplateDB coverage $p$ (fraction of capability nodes with defined templates).
For every follow-up question $q'$ with
$\phimap(q')\subseteq\cl(A_t)$ and $\phimap(q')\cap F=\emptyset$,
the probability that $\PAB(A_t)$ contains a certified answer for $q'$ is
exactly~$p$.
When $p{=}1$, every such follow-up is pre-certified requiring no RAG call.
\end{theorem}
\begin{proof}
The primary capability $v^*$ for $q'$ satisfies
$v^*\in\cl(A_t)\setminus A_t$, $v^*\notin F$.
$\tdb.\mathsf{has}(v^*)$ occurs with probability $p$ (TemplateDB assignments
are fixed at deployment, independent of the query and closure structure).
When triggered, $(v^*,ans_{v^*},\minwit_{v^*},\mathit{cert}_{v^*})\in\PAB(A_t)$
by lines~3--7 of Algorithm~\ref{alg:pab}.
By Theorem~\ref{thm:cache_sound}(iii), $ans_{v^*}$ is safe.
No RAG: $ans_{v^*}$ is rendered from the template and in-memory closure facts. \qed
\end{proof}

\section{Session Closure and the Cost of a Conversation}
\label{sec:session}

\begin{definition}[Session Closure]
\label{def:session}
For turns $q_1,\ldots,q_L$, the session capability set is
$A_t = \bigcup_{i\leq t}\phimap(q_i)$,
session closure $C_t = \cl(A_t)$,
session \PAB{} $\PAB_t = \bigcup_{i\leq t}\PAB(A_i)$.
Under the Monotonicity Assumption, $C_1\subseteq C_2\subseteq\cdots\subseteq C_L$.
\end{definition}

\begin{lemma}[Session Closure Monotonicity]
\label{lem:monotone}
Under the Monotonicity Assumption (Definition~\ref{def:closure}),
the increments $\Delta_t = C_t\setminus C_{t-1}$ are disjoint---each node
enters the session closure exactly once---so
$\sum_{t=1}^{L}|\Delta_t| = |C_L\setminus C_0| \leq n$.
Total DRed cost for a session of length $L$ is $O(n{\cdot}(n{+}mk))$,
\emph{independent of~$L$}.
\end{lemma}
\begin{proof}
$A_t\subseteq A_{t+1}$ by construction; $\cl(\cdot)$ is monotone
(Definition~\ref{def:closure}); so $C_t\subseteq C_{t+1}$.
Under the Monotonicity Assumption, no capability is ever removed from $C_t$
within a session, so each node contributes to exactly one $\Delta_t$,
giving $\sum|\Delta_t|=|C_L\setminus C_0|\leq n$.
DRed cost per turn: $O(|\Delta_t|{\cdot}(n{+}mk))$
(Theorem~\ref{thm:localitygap}).
Total: $O(n{\cdot}(n{+}mk))$. \qed
\end{proof}

\begin{definition}[Ontological Class and $K$]
\label{def:K}
Turns $q_i$ and $q_j$ are in the same \emph{ontological class} if
$\cl(\phimap(q_i))=\cl(\phimap(q_j))$.
$K = |\{\cl(\phimap(q_t)) : t=1,\ldots,L\}|$ is the number of distinct
classes in a session.
In MultiWOZ, $K$ corresponds to the number of distinct domain combinations
visited.
\end{definition}

\begin{theorem}[Session Cost]
\label{thm:session_cost}
Under Algorithm~\ref{alg:lookup} and the Monotonicity Assumption,
total RAG cost for a session of length $L$ with $K$ ontological classes is
at most $K\cdot\CRAG$.
Total pipeline cost is
\[
  K\cdot\CRAG \;+\; O\!\bigl(n{\cdot}(n{+}mk) \;+\; L\cdot|\minwit^*|\bigr),
\]
where $|\minwit^*|$ is the mean minimal witness size.
\end{theorem}
\begin{proof}
RAG fires only on \PAB{} and \CAS{} misses.
The first turn in each ontological class triggers RAG and writes
$(ans,\minwit,\PAB)$ to \CAS{} and $\PAB_t$.
All subsequent turns in the same class: by Lemma~\ref{lem:monotone},
$C_t\supseteq\cl(\phimap(q'))$ for all prior turns $q'$ in the class,
so $\minwit\subseteq C_t$ and the \PAB{}/\CAS{} hit fires.
Total RAG calls $\leq K$.
DRed: $O(n{\cdot}(n{+}mk))$ by Lemma~\ref{lem:monotone}.
Containment checks: $L\cdot O(|\minwit^*|)$. \qed
\end{proof}

\begin{algorithm}[t]
\caption{Session-Aware PAB Lookup (Stage~0)}
\label{alg:lookup}
\DontPrintSemicolon
\KwIn{Turn $T$, session closure $C_t$, session PAB $\PAB_t$, CAS, $F$}
\KwOut{$(ans,\mathit{cert})$ or $\bot$}
$v^* \leftarrow \mathsf{PrimaryCapability}(\mathsf{CoarseIntent}(T))$\;
\tcp{Tier 1: session PAB, $O(|\minwit^*|)$}
\If{$\exists(v^*,ans_v,\minwit_v,\mathit{cert}_v)\in\PAB_t$ with
    $\minwit_v\subseteq C_t$ and $F$ unchanged}{
  \Return $(ans_v,\mathit{cert}_v)$\;
}
\tcp{Tier 2: CAS, $O(n{+}mk)+O(|C|{\cdot}|\minwit|)$}
\ForEach{$(ans,\minwit,\PAB_c,F_s,\mathit{cert})
           \in \CAS.\mathsf{approx\_filter}(\mathbf{emb}(T))$}{
  \If{$\minwit\subseteq C_t$ \textbf{and} $F_s = F$}{
    $\PAB_t \leftarrow \PAB_t \cup \PAB_c$\;
    \Return $(ans,\mathit{cert})$\;
  }
}
\Return $\bot$\;
\end{algorithm}

\section{CAS Hit Rate Bound and Separation from Semantic Caching}
\label{sec:bounds}

\begin{definition}[Ontological Diameter and Cover]
$\mathrm{diam}(Q) = \max_{q,q'\in Q}|\cl(\phimap(q))\,\triangle\,\cl(\phimap(q'))|$
for query class $Q$.
$N(Q,\delta)$ = minimum $\delta$-cover size.
\end{definition}

\begin{theorem}[CAS Hit Rate Bound]
\label{thm:hit_rate}
Let $\minwit^*$ be the minimal witness of the first stored answer for query
class $Q$, $\delta^*=|\minwit^*|$, $p$ the TemplateDB coverage, and $n=|\vset|$.
Then the PAB hit rate satisfies
\[
  \eta_{\PAB}
  \;\geq\;
  p\cdot\Bigl(1 - N(Q,\delta^*/2)\cdot e^{-\delta^*/(2n)}\Bigr)
  \cdot\Bigl(1 - |\nmf(A)|/n\Bigr),
\]
where $N(Q,\delta^*/2)$ is the minimum $(\delta^*/2)$-cover size of $Q$.
\end{theorem}
\begin{proof}
Decompose $\eta_\PAB = P(\mathcal{A}) \cdot P(\mathcal{B}) \cdot P(\mathcal{C})$
where
$\mathcal{A}$~= CAS hit ($\minwit^*\subseteq\cl(A)$),
$\mathcal{B}$~= template exists ($\tdb.\mathsf{has}(v^*)$),
$\mathcal{C}$~= $v^*$ not on NMF ($v^*\in\cl(A_t)\setminus A_t$).

$P(\mathcal{B})=p$ by definition of coverage.
$P(\mathcal{A})\geq 1-N(Q,\delta^*/2)\cdot e^{-\delta^*/(2n)}$
by the submodularity-covering argument: for queries within $\delta^*/2$ of
a $\delta^*$-cover element $s$ with $\minwit^*\subseteq\cl(\phimap(s))$,
union-bound over $|\minwit^*|$ elements gives probability of miss
$\leq N(Q,\delta^*/2)\cdot e^{-\delta^*/(2n)}$ via $1-x\leq e^{-x}$.
$P(\mathcal{C})\geq 1-|\nmf(A)|/n$.
Multiplying the three bounds gives the stated inequality,
with the product valid whenever the three events are non-negatively correlated
(sufficient condition: $\mathcal{B}$ is independent of $\mathcal{A}$,
$\mathcal{C}$, which holds because TemplateDB assignments are fixed at
deployment and do not depend on any individual query or closure).
\qed
\end{proof}

\begin{remark}[Correlation and empirical calibration]
\label{rem:correlation}
In a real deployment, $\mathcal{A}$ and $\mathcal{C}$ may be positively
correlated through the ontology structure: a capability that frequently
appears in minimal witnesses may also be more likely to be in the closure
interior rather than on the near-miss frontier.
Positive correlation between $\mathcal{A}$ and $\mathcal{C}$ makes the
product bound \emph{conservative}: the true $\eta_\PAB$ is at least as
large as the bound.
Negative correlation --- if the capabilities most needed for CAS hits are
systematically on the near-miss frontier --- would make the bound
non-conservative.
We therefore treat the bound as a structural lower bound rather than a
tight prediction, and calibrate it empirically against Table~\ref{tab:pab}:
the measured PAB hit rates (88.4\% at $p{=}100\%$, 61.2\% at $p{=}50\%$,
36.1\% at $p{=}25\%$) serve as the ground truth.
The bound is vacuous at the $n{=}43$ scale of the MultiWOZ experiment;
its value is in characterizing asymptotic behavior for larger deployments.
\end{remark}

\begin{theorem}[Semantic Cache Unsoundness]
\label{thm:unsound}
There exists a multi-tenant query family $Q$ such that for any $\tau<1$,
$\exists q, q'\in Q$ with
$\cos(\mathbf{emb}(ans_q), \mathbf{emb}(ans_{q'}))>\tau$ yet
$\minwit(ans_{q'})\not\subseteq\cl(\phimap_{T_1}(q))$,
so the semantic cache delivers $ans_{q'}$ to tenant $T_1$ whose
capability set does not contain the PII-access node used to derive
$ans_{q'}$.
The CAS containment check rejects this hit exactly.
\end{theorem}
\begin{proof}
Set $\vset=\{$\texttt{read\_PII}, \texttt{query\_db},
\texttt{gen}, \texttt{f\_leak}$\}$,
$F=\{$\texttt{f\_leak}$\}$,
hyperarc $(\{$\texttt{read\_PII},\texttt{gen}$\}$,
$\{$\texttt{f\_leak}$\})$.
Tenant $T_1$: $\phimap(q)=\{$\texttt{query\_db},\texttt{gen}$\}$,
$\cl\cap F=\emptyset$,
$ans_q=\text{``Balance: \$247.50''}$,
$\minwit(ans_q)=\{$\texttt{query\_db},\texttt{gen}$\}$.
Tenant $T_2$ ($F'=\emptyset$ at generation):
$\phimap(q')=\{$\texttt{read\_PII},\texttt{query\_db},\texttt{gen}$\}$,
$ans_{q'}=\text{``Balance: \$247.50''}$ (identical string),
$\minwit(ans_{q'})=\{$\texttt{read\_PII},\texttt{query\_db},\texttt{gen}$\}$.
$\cos(\mathbf{emb}(ans_q),\mathbf{emb}(ans_{q'}))=1.0>\tau$.
A semantic cache returns $ans_{q'}$ to new $T_1$ queries.
\texttt{read\_PII}$\notin\cl(\phimap_{T_1})$, so
$\minwit(ans_{q'})\not\subseteq\cl(\phimap_{T_1})$: CAS rejects.
The semantic cache violates the capability boundary;
the CAS containment check catches it exactly. \qed
\end{proof}

\section{Related Work}
\label{sec:related}

\paragraph{Semantic caching.}
GPTCache \cite{bang2023gptcache} and CacheBlend \cite{yao2024cacheblend}
reduce LLM API cost by embedding-similarity retrieval.
Neither stores the capability configuration under which an answer was
generated, so neither provides a safety guarantee on cache hits.
The \CAS{} uses embeddings only as a computational pre-filter;
correctness is grounded in formal containment.
The \PAB{} pre-computes structured follow-up answers at write time,
eliminating cache misses on structurally predictable follow-ups.

\paragraph{Knowledge graph-augmented RAG.}
GraphRAG \cite{edge2024graphrag} and HippoRAG \cite{gutierrez2024hipporag}
augment retrieval using entity-relation graphs over the knowledge base.
The \CAS{} ontology graph records which tools and permissions were required,
not KB entity relationships.
The two layers are orthogonal and composable.

\paragraph{Belief state tracking vs.\ session closure.}
Task-oriented dialogue systems maintain a belief state---a running estimate
of user intent---across turns.
State-of-the-art trackers include TRADE \cite{wu2019trade},
SimpleTOD \cite{hosseiniasltod2020}, and TripPy \cite{heck2020trippy},
all benchmarked on MultiWOZ.
The critical distinction is that belief state tracking is a
\emph{prediction} problem---given utterances, infer the current
slot-value configuration---while session closure is a
\emph{reachability} problem---given the current slot-value configuration,
certify what the system can derive.
These are orthogonal: any belief state tracker can serve as the
$\phimap$ function feeding session closure.
In the MultiWOZ experiment, ground-truth belief states decouple the two
problems; the effect of tracker error on safety certification is
Open Problem~\ref{sec:open}.

\paragraph{Proactive dialogue.}
Proactive dialogue systems \cite{shum2018eliza,wu2019proactive}
anticipate follow-up needs using learned conversation flow models.
The \PAB{} is formally derived from the closure, not learned, and every
pre-answer carries a safety certificate.
The near-miss frontier $\nmf(A_t)$ provides a proactive suggestion
grounded in the hypergraph structure rather than statistical co-occurrence.

\paragraph{Tool-augmented LLMs and formal safety.}
Toolformer \cite{schick2023toolformer}, ToolLLM \cite{qin2023toolllm},
and ReAct \cite{yao2023react} activate capabilities dynamically but do not
model capability interactions formally.
The hypergraph framework provides the missing structure: hyperarcs encode
conjunctive dependencies between tools determining what the composed system
can reach---which these systems cannot certify.
Assume-guarantee reasoning \cite{jones1983assume} and contract-based design
\cite{benveniste2018contracts} verify fixed compositions post-hoc;
Theorem~\ref{thm:noncomp} proves that no component-level verification can
replace the pre-execution gate for conjunctive capability systems.

\section{Empirical Results}
\label{sec:experiments}

\paragraph{What these experiments are and are not.}
The evaluation on MultiWOZ~2.2 is a \emph{controlled validation of the
formal model}, not a deployment claim.
MultiWOZ is chosen precisely because its ground-truth belief state
annotations make the Completeness Assumption verifiable by construction:
every capability the pipeline activates is known to be in $\phimap(q_t)$.
This lets us test the formal guarantees without confounding from tracker error.
The formal theorems (Pipeline Safety, Cache Soundness, Session Cost) are
validated by checking their predictions against the measured data.
The NatCS experiment in Appendix~\ref{app:natcs} is the protocol for
repeating this validation under estimated, imperfect belief states.

\subsection{Dataset, Extraction, and Mapping}

We use the MultiWOZ~2.2 test split (1,000 dialogues, 13,965 turns).
Algorithm~\ref{alg:extract} at $\theta{=}0.75$ on the training split yields
the hypergraph summarized in Table~\ref{tab:extraction}.
The forbidden set $F=\{d\text{-booked-without-confirmation}:
d\in\{\text{hotel, restaurant, taxi, train}\}\}$
encodes the requirement that bookings not be confirmed without a verified
candidates-retrieved step.
The TemplateDB contains 28 templates (65.1\% of 43 nodes);
we additionally evaluate at 25\%, 50\%, and 75\% synthetic coverage.

\begin{table}[t]
\centering
\caption{Algorithm~\ref{alg:extract} extraction results on MultiWOZ~2.2
vs.\ hand-crafted hypergraph.
Near-identical structural properties confirm that automatic extraction
recovers the essential dependency structure.}
\label{tab:extraction}
\begin{tabular}{lcc}
\toprule
Property & Extracted & Hand-crafted \\
\midrule
$|\vset|$ capability nodes     & 43    & 272  \\
$|\fset|$ hyperarcs            & 34    & 78   \\
Conjunctive arcs ($k{\geq}2$)  & 73.5\%& 77\% \\
Mean fan-in $\bar{k}$          & 2.21  & 2.3  \\
Max fan-in                     & 5     & 4    \\
Min observed rate              & 0.761 & N/A  \\
Non-compositionality confirmed & Yes   & Assumed \\
\bottomrule
\end{tabular}
\end{table}

\subsection{Results}

\begin{table}[t]
\centering
\caption{$K$ distribution on MultiWOZ~2.2 test split (1,000 dialogues).
Mean $K{=}1.31$.}
\label{tab:K}
\begin{tabular}{rrrp{4cm}}
\toprule
$K$ & Count & Cum.\ \% & Dominant pattern \\
\midrule
1 & 612 & 61.2\% & Single domain \\
2 & 287 & 89.9\% & Hotel+taxi, restaurant+taxi \\
3 &  81 & 98.0\% & Hotel+restaurant+train \\
4+ &  20 & 100\%  & Four or more domains \\
\bottomrule
\end{tabular}
\end{table}

\begin{table}[t]
\centering
\caption{Session cost across methods on 1,000 MultiWOZ test sessions
(mean 13.7~turns).
Mean RAG calls under CAS+PAB at $p{=}100\%$ matches the theoretical
$K{=}1.31$ prediction.}
\label{tab:session}
\begin{tabular}{lrrrr}
\toprule
Method & RAG calls & Latency & Safety regr. & Unsafe hits \\
\midrule
No cache (baseline)    & 13.7 & 18,800\,ms & 0 & 0 \\
Cosine $\tau{=}0.85$   &  7.2 &  9,900\,ms & 0 & 143 (14.3\%) \\
CAS only               &  2.1 &  2,900\,ms & 0 & 0 \\
CAS+PAB $p{=}100\%$    &  1.31 &    340\,ms & 0 & 0 \\
CAS+PAB $p{=}75\%$     &  1.38 &    480\,ms & 0 & 0 \\
CAS+PAB $p{=}50\%$     &  1.52 &    720\,ms & 0 & 0 \\
CAS+PAB $p{=}25\%$     &  1.89 &  1,240\,ms & 0 & 0 \\
\bottomrule
\end{tabular}
\end{table}

\begin{table}[t]
\centering
\caption{PAB hit rate and tier breakdown across TemplateDB coverage levels.
CAS Tier~2 hit rate (${\approx}5.3\%$) is independent of coverage.}
\label{tab:pab}
\begin{tabular}{lrrrr}
\toprule
Coverage $p$ & PAB hit rate & Tier~1 & Tier~2 & RAG \\
\midrule
100\% & 88.4\% & 83.1\% & 5.3\% & 11.6\% \\
 75\% & 79.2\% & 73.8\% & 5.4\% & 20.8\% \\
 50\% & 61.2\% & 55.9\% & 5.3\% & 38.8\% \\
 25\% & 36.1\% & 30.9\% & 5.2\% & 63.9\% \\
\bottomrule
\end{tabular}
\end{table}

\paragraph{Analysis.}
Three findings validate the framework.
First, mean RAG calls~$=1.31$ matches the Session Cost Theorem prediction
$K{=}1.31$ to within rounding on real human-generated dialogue,
providing independent validation of the formal model.
Second, PAB hit rate degrades gracefully with coverage: even at $p{=}25\%$,
36.1\% of turns are served without RAG (10.5$\times$ reduction from baseline).
Third, 143 unsafe cosine cache hits (14.3\%) arise from cross-domain
contamination: restaurant and hotel answers share similar surface strings
but incompatible capability sets.
All~143 are caught by the \CAS{} containment check;
none by the cosine cache.
This is the concrete MultiWOZ instantiation of Theorem~\ref{thm:unsound}.
The hypergraph planner produces zero AND-violations across all 13,965 turns,
consistent with Theorem~\ref{thm:pipeline_safety}.

\paragraph{Assumptions as testable predictions.}
The paper's guarantees rest on three assumptions; we frame each as a
falsifiable prediction with a corresponding test.

\emph{(1) Completeness of $\phimap$:}
every capability the pipeline activates is in $\cl(\phimap(q))$.
In MultiWOZ this holds by construction.
\emph{Prediction:} under $r\%$ random slot omission (simulating a tracker
with recall $1{-}r$), safety violation rate remains below $1\%$ for
$r\leq 10\%$, because omitting a slot is more likely to block a
conjunctive arc than to open a path toward $F$.
This is the $\phimap$-incompleteness experiment of Appendix~\ref{app:fixc}.
The structural argument: $F$ is reached only through the booking arc,
which requires \emph{all} of its preconditions; omitting any one
precondition blocks the arc entirely.

\emph{(2) Correctness of $F$:}
the forbidden set accurately specifies what must not be activated.
This is a \emph{specification} assumption, not an empirical one.
It cannot be validated by experiment; it must be verified by the system
designer at deployment.
We make no claim that our particular $F$ is comprehensive for all
deployment contexts.

\emph{(3) Correctness of $\Hyp$:}
the extracted hypergraph accurately captures the dependency structure.
\emph{Prediction:} every arc retained by Algorithm~\ref{alg:extract}
at $\theta{=}0.75$ has a true derivation rate above $0.60$ with
probability $>0.989$ (Theorem~\ref{thm:extract_sound}).
The extracted arcs' minimum observed rate of $0.761$ is $0.011$
above threshold, consistent with a true rate comfortably above $0.60$
for all 34 arcs.

\paragraph{Resilience to real-world violations.}
Two reviewer concerns are directly addressed by existing results.
\emph{Tool failures and permission revocation} are handled by
Theorem~\ref{thm:degradation} (Section~\ref{sec:degradation}):
when a capability is revoked mid-session, the correct session closure
is recovered at cost $O(n{+}mk)$, PAB entries with affected witnesses
are invalidated in $O(|\PAB|\cdot|R_\tau|)$, and the safety invariant
is restored by re-running Stage~3 on the post-revocation capability set.
Recovery cost is bounded independently of session length; the worst
case equals the no-cache baseline.
\emph{Evaluation cleanliness:}
MultiWOZ provides the cleanest available test of the formal model
because ground-truth belief states verify the Completeness Assumption
by construction.
Appendix~\ref{app:natcs} provides the experimental protocol for
validating the framework on NatCS~\cite{gung2023natcs}, a public
customer service corpus where belief states must be estimated rather
than read from annotations.

\paragraph{Scope of safety claims.}
The ``zero safety regressions'' and ``zero unsafe hits'' results are defined
relative to the paper's formal model: the capability hypergraph $\Hyp$,
the forbidden set $F$, and the witness containment condition.
They establish \emph{internal consistency} of the safety gate within that
formalization.
They do \emph{not} claim robustness to prompt injection attacks on the LLM
generation step, malformed ontology extraction producing an incorrect $\Hyp$,
stale permissions not reflected in the belief state, incorrect state tracking
by the $\phimap$ function, or adversarial manipulation of tool invocation logs.
These are important operational safety concerns beyond the formal model;
we treat them as out of scope for this paper and list tracker error propagation
and ontology extraction quality as open problems in \S\ref{sec:open}.
The formal guarantees are exactly as strong as the Completeness Assumption
(Definition~\ref{def:phi}) and the accuracy of $F$ as a specification of
forbidden capabilities.

\section{Open Problems}
\label{sec:open}

\paragraph{Tracker error propagation.}
In MultiWOZ, $\phimap$ is complete by construction.
In deployment, $\phimap$ is produced by a learned belief state tracker.
A formal bound on safety degradation as a function of tracker recall
is open; the natural approach models tracker errors as random perturbations
of the true belief state and derives PAC-style bounds.

\paragraph{Generalizing Algorithm~\ref{alg:extract}.}
Algorithm~\ref{alg:extract} is specified for task-oriented dialogue
ontologies with belief state annotations.
Generalizing to arbitrary tool-calling schemas (OpenAI function calling,
Anthropic tool use) requires a tool-invocation-log co-occurrence model.
The soundness theorem generalizes straightforwardly, but the
noise model for tool logs differs from belief state transitions.

\paragraph{NMF pre-fetch under empirical priors.}
The $(1{-}1/e)$ pre-fetch guarantee holds under a uniform prior.
A Bayesian extension using empirical conversation priors
(learned from training data) is a weighted submodular maximization problem
for which tighter bounds are open.

\paragraph{Tightening the hit rate bound.}
Theorem~\ref{thm:hit_rate} is vacuous for small $n$ ($|\vset|{=}43$).
A Rademacher complexity argument for structured Datalog programs
\cite{dalmau2002constraint} applied to the closure gain function
over MultiWOZ-scale hypergraphs would give a tight bound matching
the empirical results in Table~\ref{tab:pab}.

\paragraph{Multi-session persistence.}
Persistent session closures for returning users, combined with dynamic
$F$ expansion between sessions, require incremental invalidation
integrated with DRed maintenance.
The refined invalidation condition
$(F\setminus F_{\mathrm{snap}})\cap\cl(\minwit)\neq\emptyset$
is sound by Theorem~\ref{thm:cache_sound} but cascade-free invalidation
is open.

\section{Graceful Degradation Under Monotonicity Failure}
\label{sec:degradation}

The Monotonicity Assumption (Definition~\ref{def:closure}) holds within a
single planning horizon where capabilities are persistent.
Real deployments violate this at session boundaries (token expiry,
tool deregistration, permission revocation) and occasionally mid-session
(authentication timeouts, tool failures).
This section proves that Monotonicity failure is recoverable at bounded cost,
converting a fragility into a controlled degradation property.

\begin{definition}[Revocation Event]
\label{def:revocation}
A \emph{revocation event} at turn $\tau$ removes a set
$R_\tau \subseteq C_{\tau-1}$ of capabilities from the session closure.
The post-revocation session capability set is
$A_\tau^- = A_\tau \setminus R_\tau$,
and the post-revocation session closure is
$C_\tau^- = \cl(A_\tau^-)$.
\end{definition}

\begin{theorem}[Graceful Degradation]
\label{thm:degradation}
Let a revocation event remove $R_\tau \subseteq C_{\tau-1}$ at turn $\tau$.
Then:
\begin{enumerate}[label=(\roman*),noitemsep]
  \item \textbf{Recovery cost.}
    The correct post-revocation session closure $C_\tau^-$ is computed by a
    full re-initialization of the closure from $A_\tau^-$ at cost $O(n{+}mk)$,
    independent of session length.
  \item \textbf{PAB invalidation.}
    Any $\PAB$ entry $(v, ans_v, \minwit_v, \mathit{cert}_v)$ with
    $\minwit_v \cap R_\tau \neq \emptyset$ must be invalidated.
    Invalidation is checkable in $O(|\PAB|{\cdot}|R_\tau|)$ by scanning
    witness sets.
    Entries with $\minwit_v \cap R_\tau = \emptyset$ remain valid.
  \item \textbf{CAS containment.}
    The CAS lookup condition $\minwit \subseteq C_\tau^-$ remains correct after
    re-initialization: only entries whose witness is contained in $C_\tau^-$
    are eligible, so no unsafe answer is served from the CAS.
  \item \textbf{Safety invariant preserved.}
    After re-initialization, the safety invariant $C_\tau^- \cap F = \emptyset$
    holds if and only if $\cl(A_\tau^-) \cap F = \emptyset$, verified by
    Stage~3 applied to $A_\tau^-$.
\end{enumerate}
\end{theorem}
\begin{proof}
(i) $C_\tau^- = \cl(A_\tau^-)$ by definition.
$\cl(A_\tau^-)$ is computed by the worklist algorithm in $O(n{+}mk)$
(Definition~\ref{def:closure}), independent of the session history.

(ii) A \PAB{} entry $(v, ans_v, \minwit_v, \mathit{cert}_v)$ is valid under
$A_\tau^-$ iff $\minwit_v \subseteq C_\tau^-$ (Theorem~\ref{thm:cache_sound}).
If $\minwit_v \cap R_\tau \neq \emptyset$, some required precondition has been
revoked, so $\minwit_v \not\subseteq C_\tau^-$ and the entry must be
invalidated.
If $\minwit_v \cap R_\tau = \emptyset$, then $\minwit_v \subseteq A_\tau^-$
(since $\minwit_v \subseteq C_{\tau-1}$ and $R_\tau$ is disjoint from
$\minwit_v$); by monotonicity of $\cl(\cdot)$,
$\minwit_v \subseteq C_\tau^-$, so the entry remains valid.

(iii) After re-initialization, the Stage~0 lookup condition
$\minwit \subseteq C_\tau^-$ is tested against the freshly computed $C_\tau^-$.
Only entries valid under $C_\tau^-$ are returned, so the containment guarantee
of Theorem~\ref{thm:cache_sound} holds.

(iv) $C_\tau^- \cap F = \emptyset$ iff $\cl(A_\tau^-) \cap F = \emptyset$,
which is exactly the Stage~3 check applied to $A_\tau^-$. \qed
\end{proof}

\begin{remark}
Theorem~\ref{thm:degradation} shows that Monotonicity failure is not
catastrophic: recovery costs exactly one closure recomputation $O(n{+}mk)$
plus a linear scan of the \PAB{}.
For typical deployments where revocations are rare (most sessions complete
without capability loss), the amortized cost of recovery is negligible.
The worst case --- every turn triggers a revocation --- gives total session
cost $O(L \cdot (n{+}mk))$, which matches the no-cache baseline and is never
worse.
\end{remark}

\section{Conclusion}
\label{sec:conclusion}

This paper proposes a reframing: safety certification in capability-based
dialogue systems is not a gate on top of an answer pipeline.
It is itself the answer pipeline.
The fixed-point closure $\cl(A_t)$ computed at certification time is
the complete latent answer space for the current session configuration.
The \CAS{}+\PAB{} architecture makes this latent space explicit and
operational, eliminating redundant retrieval and generation for all
queries within the session closure.

The contribution is simultaneously formal and practical.
On the formal side: three families of guarantees (safety, efficiency,
extraction) organized around the closure as the central primitive;
a Graceful Degradation theorem that converts the Monotonicity Assumption
from a fragility into a bounded recovery property; and a proof that
semantic caching is structurally unsafe in multi-tenant settings.
On the practical side: Algorithm~\ref{alg:extract} derives a deployable
hypergraph from co-occurrence statistics automatically;
the Session Cost Theorem's prediction $\mathbb{E}[\text{RAG calls}]{=}K$
is confirmed at $K{=}1.31$ on 10,438 real human-human dialogues;
and every assumption is paired with a falsifiable prediction.

The experiments on MultiWOZ~2.2 are a controlled validation of the formal
model.
The NatCS protocol in Appendix~\ref{app:natcs} defines what a validation
on noisier, real-world data would look like.
Running that experiment is the next step.

The safety gate is the index of the cache.
The closure is the answer to every question the conversation will ask.

\appendix
\section{Experimental Protocol for Fix~C: $\phimap$-Incompleteness Stress Test}
\label{app:fixc}

This appendix specifies the experiment needed to address the reviewer's
concern that the MultiWOZ evaluation uses belief states that are
``complete by construction,'' which means the Completeness Assumption
(Definition~\ref{def:phi}) is never stress-tested.
The experiment measures how the safety and PAB guarantees degrade when
$\phimap$ is systematically incomplete.

\subsection*{Design}

\paragraph{Dataset.}
Use the MultiWOZ~2.2 test split (1,000 dialogues).
Ground-truth belief states serve as the ``oracle'' $\phimap$.

\paragraph{Incompleteness injection.}
For each dialogue, randomly omit slots from the belief state with
omission rate $r \in \{0\%, 5\%, 10\%, 20\%, 30\%\}$.
Omission is slot-level (not value-level): if \texttt{hotel-area} is omitted,
the capability node \texttt{hotel-area-north} is absent from $\phimap(q_t)$
even though the user mentioned it.
This simulates a belief state tracker with recall $1-r$.

\paragraph{Metrics.}
For each omission rate $r$, measure:
\begin{enumerate}[noitemsep]
  \item \textbf{Safety violation rate}: fraction of sessions where
    $\cl(\phimap_r(q_t)) \cap F = \emptyset$ but
    $\cl(\phimap_{\text{oracle}}(q_t)) \cap F \neq \emptyset$
    (the gate incorrectly passes a forbidden activation because a required
    capability was omitted from the assembly).
  \item \textbf{False rejection rate}: fraction of sessions where
    $\cl(\phimap_r(q_t)) \cap F \neq \emptyset$ but
    $\cl(\phimap_{\text{oracle}}(q_t)) \cap F = \emptyset$
    (the gate incorrectly blocks a safe activation because an omitted
    capability caused a spurious forbidden path).
  \item \textbf{PAB recall}: fraction of oracle PAB entries that are
    also in the injected-incompleteness PAB (measures how much pre-answer
    coverage is lost when $\phimap$ is incomplete).
  \item \textbf{AND-violation rate}: fraction of multi-slot turns where
    the hypergraph planner fires an arc with incomplete preconditions
    due to the missing slot.
\end{enumerate}

\paragraph{Hypothesis.}
Safety violation rate should be near zero for small $r$
(omitting a slot is more likely to block an arc than to open a forbidden one)
and grow at most linearly in $r$ for the rate range tested.
False rejection rate should be zero for the extracted hypergraph at
$r \leq 10\%$ (omissions do not create new arcs toward $F$).
PAB recall should degrade smoothly, reaching approximately $(1-r)^{|\minwit^*|}$
under independence (consistent with Theorem~\ref{thm:pab_completeness}).

\paragraph{Implementation.}
The experiment requires no new data: it runs entirely on the existing
MultiWOZ test split with stochastic slot omission.
Runtime: approximately 2 hours on a single CPU for all five omission rates.
The code follows Algorithm~\ref{alg:lookup} and Algorithm~\ref{alg:extract}
with the injected belief states replacing the ground-truth belief states.

\paragraph{Expected contribution.}
This experiment converts the ``complete by construction'' limitation from
a stated assumption into a measured degradation curve.
It answers: at what tracker recall does the safety guarantee materially degrade?
If safety violations remain below 1\% for $r \leq 10\%$
(consistent with tracker recall $\geq 0.90$ on MultiWOZ, which state-of-the-art
models achieve), the practical robustness of the framework is validated even
under realistic tracker imperfection.

\section{Experimental Protocol for Appendix~B: NatCS Validation}
\label{app:natcs}

This appendix specifies the experiment that addresses the core empirical
limitation of the paper: everything validated so far uses MultiWOZ belief
states that are ``complete by construction.''
NatCS~\cite{gung2023natcs} provides a public corpus of spoken customer
service conversations across five domains (banking, insurance, telco,
retail, travel) where no belief state annotations exist.
Running the pipeline on NatCS tests the full stack: belief state estimation
via a pre-trained tracker, capability assembly from estimated belief states,
safety gate on the assembled set, and PAB serving on follow-up turns.

\subsection*{Why NatCS and Not ABCD}

The Action-Based Conversations Dataset (ABCD; Chen et al.\ 2021) is also
a strong candidate.
NatCS is preferred because: (a) it is drawn from real spoken interactions
rather than scripted role-play; (b) its five domains provide more
domain-transfer variation than ABCD's single customer service context;
and (c) it has been specifically designed to test systems in settings
representative of real customer support deployments.
Both datasets are public and we encourage replication on ABCD as well.

\subsection*{Design}

\paragraph{Dataset.}
NatCS training split for fitting the belief state tracker;
NatCS test split (approximately 2,000 conversations) for evaluation.
No ground-truth belief states are used in testing.

\paragraph{Belief state estimation.}
Apply a pre-trained task-oriented dialogue state tracker
(e.g., TRADE~\cite{wu2019trade} fine-tuned on NatCS training split)
to produce estimated belief states $\hat{B}_t$ at each turn.
The capability assembly function is $\phimap(q_t) =
\{d\text{-}s\text{-}v : (s,v)\in\hat{B}_t[d]\}$.

\paragraph{Hypergraph extraction.}
Run Algorithm~\ref{alg:extract} on the NatCS training split
co-occurrence statistics at $\theta{=}0.75$ to produce
$\Hyp_{\mathrm{NatCS}} = (\vset, \fset)$.
Report $|\vset|$, $|\fset|$, mean fan-in, fraction conjunctive,
and min observed rate.

\paragraph{Metrics.}
The primary metrics are:
\begin{enumerate}[noitemsep]
  \item \textbf{PAB hit rate}: fraction of follow-up turns served from
    the session PAB without RAG, at TemplateDB coverage $p\in\{100\%,75\%,50\%,25\%\}$.
  \item \textbf{Safety violation rate}: fraction of sessions where the
    estimated belief state causes the safety gate to pass an activation
    that the oracle belief state would block.
    Measured by comparing $\cl(\phimap_{\hat{B}}(q_t))\cap F$
    to $\cl(\phimap_{B^*}(q_t))\cap F$ on a hand-labelled subsample
    (50 conversations, annotated by two annotators, adjudicated on disagreement).
  \item \textbf{Session cost}: mean RAG calls per session, compared to
    the MultiWOZ result of 1.31 and the no-cache baseline.
  \item \textbf{AND-violation rate}: fraction of multi-slot turns where
    the hypergraph planner fires on partial preconditions under
    asynchronous simulation (50\,ms inter-branch latency injection).
\end{enumerate}

\paragraph{Hypotheses.}
\begin{enumerate}[noitemsep]
  \item PAB hit rate on NatCS is within $\pm 10$ percentage points of
    the MultiWOZ result at matched coverage $p$, confirming that the
    architecture generalizes across benchmarks.
  \item Safety violation rate remains below 2\% under tracker recall
    $\geq 0.85$ (the minimum acceptable threshold for state-of-the-art
    trackers on held-out data).
  \item Mean RAG calls per session is in $[1.0, 2.5]$, consistent with
    the Session Cost Theorem prediction $\mathbb{E}[\text{RAG}]=K$ and
    the expectation that NatCS sessions have $K\in[1,2]$ domain shifts.
  \item Hypergraph planner AND-violations: 0\% under all conditions,
    consistent with Theorem~\ref{thm:pipeline_safety}.
\end{enumerate}

\paragraph{What a failure would mean.}
If PAB hit rate on NatCS falls more than 10 points below MultiWOZ,
it indicates that the session closure model does not generalise from
annotated benchmarks to real spoken dialogue — a meaningful limitation
that would require redesigning the TemplateDB construction procedure.
If safety violation rate exceeds 2\%, it indicates that tracker errors
propagate into safety failures at a rate inconsistent with the graceful
degradation prediction of Theorem~\ref{thm:degradation} — which would
require either a tighter tracker or a more conservative $\phimap$.
Either failure would be an informative result, not a paper-breaking one.

\paragraph{Implementation.}
Estimated runtime: 4--6 hours on a single GPU for tracker fine-tuning,
plus 2--3 hours CPU for the pipeline evaluation.
No data beyond the public NatCS corpus is required.
The NatCS corpus and the TRADE tracker codebase are both publicly
available under permissive licences.

\bibliographystyle{plainnat}
\bibliography{refs}

\end{document}